\documentclass{article}

\usepackage{microtype}
\usepackage{graphicx}
\usepackage{subfigure}
\usepackage{booktabs} 
\usepackage{amsmath}
\usepackage{amssymb}
\usepackage{amsthm}
\usepackage[titletoc]{appendix}

\usepackage{hyperref}

\newtheorem{prop}{Proposition}

\usepackage[accepted]{icml2019}

\icmltitlerunning{Motion Equivariant Networks for Event Cameras with the Temporal Normalization Transform}

\begin{document}

\twocolumn[
\icmltitle{Motion Equivariant Networks for Event Cameras\\with the Temporal Normalization Transform}



\icmlsetsymbol{equal}{*}

\begin{icmlauthorlist}
\icmlauthor{Alex Zihao Zhu}{phila}
\icmlauthor{Ziyun Wang}{phila}
\icmlauthor{Kostas Daniilidis}{phila}
\end{icmlauthorlist}

\icmlaffiliation{phila}{University of Pennsylvania}

\icmlcorrespondingauthor{Alex Zihao Zhu}{alexzhu@seas.upenn.edu}

\icmlkeywords{Equivariance, Event Cameras}

\vskip 0.3in
]



\printAffiliationsAndNotice{}  

\begin{abstract}
In this work, we propose a novel transformation for events from an event camera that is equivariant to optical flow under convolutions in the 3-D spatiotemporal domain. Events are generated by changes in the image, which are typically due to motion, either of the camera or the scene. As a result, different motions result in a different set of events. For learning based tasks based on a static scene such as classification which directly use the events, we must either rely on the learning method to learn the underlying object distinct from the motion, or to memorize all possible motions for each object with extensive data augmentation. Instead, we propose a novel transformation of the input event data which normalizes the $x$ and $y$ positions by the timestamp of each event. We show that this transformation generates a representation of the events that is equivariant to this motion when the optical flow is constant, allowing a deep neural network to learn the classification task without the need for expensive data augmentation. We test our method on the event based N-MNIST dataset, as well as a novel dataset N-MOVING-MNIST, with significantly more variety in motion compared to the standard N-MNIST dataset. In all sequences, we demonstrate that our transformed network is able to achieve similar or better performance compared to a network with a standard volumetric event input, and performs significantly better when the test set has a larger set of motions than seen at training.
\end{abstract}
\begin{figure}
\centering
\includegraphics[width=0.4\linewidth]{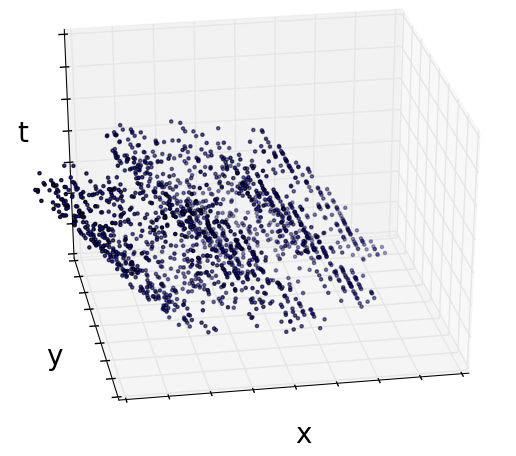}
\includegraphics[width=0.4\linewidth]{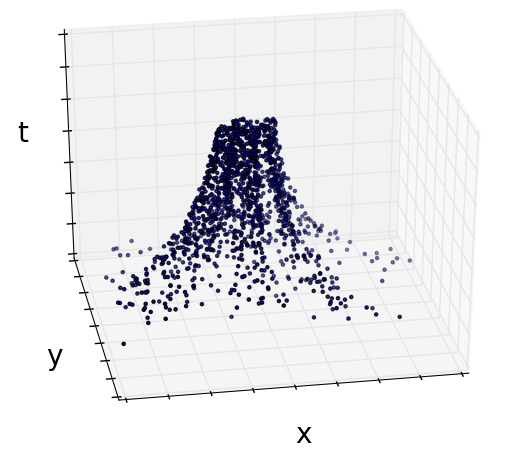}

\includegraphics[width=0.4\linewidth]{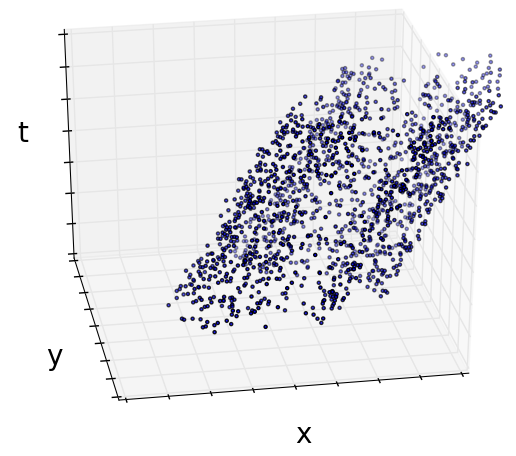}
\includegraphics[width=0.4\linewidth]{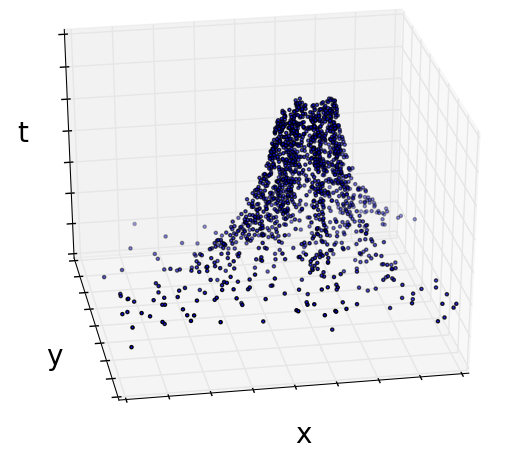}
\caption{Classical convolution layers would not be equivariant to event motions on the left, since they are shear deformations of the event volume. After transforming to canonical coordinates on the right, the volume translates uniformly, resulting in equivariance to the motion. Left: Raw input events. Right: Corresponding transformed events.}
\label{fig:transformation}
\end{figure}
\section{Introduction}
Event-based cameras are a novel asynchronous sensing modality that provides exciting benefits, such as the ability to track fast moving objects with no motion blur and low latency, high dynamic range, and low power consumption. These benefits provide a compelling reason to utilize these cameras in traditional vision tasks such as image classification, where they can operate in challenging conditions beyond the capability of traditional cameras. 

However, the data generated by these cameras, often represented as a stream of changes and their associated spatiotemporal positions, do not directly fit into the traditional paradigm for neural networks, which are designed to perform inference on 2D image frames. Recent works have tried to adapt events into this paradigm by performing convolutions over either compressed 2D representations of the events or discretized 3D volumes. However, due to the high temporal resolution of the events, this voxel grid will naturally embed the motion of the image, and so any given image has a near infinite number of possible 3D representations, depending on the motion of the camera and or scene. 

In this work, we propose a novel coordinate transformation for the 3D event data, which transforms the events into a space that is equivariant to motion for convolutions. In particular, changes in optical flow between sets of events are transformed into pure translations in the new space. Due to the equivariance of CNNs to translations, this transformation results in the CNN becoming equivariant to changes in the optical flow, such that a change in the flow corresponds to a predictable translation at each subsequent set of feature activations. This equivariance preserves the effect of changes in optical flow through each set of filters in the networks, and saves network capacity that would otherwise have been used to learn each individual motion.

Our full pipeline consists first of a landmark regressor, which uses a spatial transformer to estimate the position of a learned landmark in the image, in order to obtain translation invariance. The events are centered around this landmark, and then transformed with the temporal normalization transform to generate a motion equivariant representation, which can be passed into a standard CNN. We demonstrate that our method allows for significantly improved generalization of a classification network to motions, particularly in the case where the set of testing motions is independent to those in the training set. 

Our contributions can be summarized as:
\begin{itemize}
\item The temporal normalization transform (TNT) for events, which transforms events into a space that is equivariant to changes in optical flow for convolutions in a CNN.
\item A CNN architecture which combines a landmark regression network with the TNT to produce representations that are invariant to translation and equivariant to optical flow.
\item The N-MOVING-MNIST dataset, consisting of simulated event data over MNIST images, with many more (30) motion directions than past datasets.
\item Quantitative evaluations on both real and simulated event based datasets, including tests with few motions at training and many different motions at test time.
\end{itemize}
\section{Related Work}
Due to the high speed and dynamic range properties of event cameras, a number of works have attempted to represent the event stream in a form suitable for traditional CNNs for both classification and regression tasks. \cite{moeys2016steering, amir2017low, maqueda2018event, iacono2018towards} generate event histograms, consisting of the number of events at each pixel, and use these as images to classify the position of a robot, perform gesture recognition, estimate steering angle, and perform object classification, respectively. 

Several methods have also incorporated the event timestamps in the inputs. \cite{Zhu-RSS-18} represent the events with event histograms as well as the last timestamp at each pixel, to perform self-supervised optical flow estimation. Similarly, \cite{ye2018unsupervised} use the average timestamp at each pixel to perform unsupervised egomotion and depth estimation, and \cite{alonso2018ev} encode the events as a 6 channel image, consisting of positive and negative event histograms, timestamp means and standard deviations, in order to perform semantic segmentation. \cite{zhu2018unsupervised} introduced the discretized event volume, which discretizes the time dimension, and then inserts events using interpolation to perform unsupervised optical flow and egomotion and depth estimation. \cite{lagorce2017hots, sironi2018hats} propose the time surface, which encode the rate of events appearing at each pixel. 

In a different vein, \cite{wang2019space} treat the events as a point cloud, and use PointNet~\cite{qi2017pointnet} to process them, while \cite{sekikawa2018constant} propose a solution which learns a set of 2D convolution kernels with associated optical flow directions, which are used to deblur the events at each step of the convolution. 

However, these methods either compress the event information into the 2D space, or do not address the issue of equivariance to optical flow when representing events in 3D. However, most current event-based classification datasets are generated with only a limited subset of motions, many from a servo motor with a fixed trajectory. In addition, most datasets have the same motions in the training and test sequences. As a result, the issue of equivariance does not appear, as the network only has to memorize a small number of motions for each class.

Equivariance for CNNs is a well studied topic, and has roots in the study of Lie generators~\cite{ferraro1988relationship, segman1992canonical} and steerability~\cite{freeman1991design, simoncelli1992shiftable, teo1998design}. Recent works have extended these ideas for equivariance of CNNs to a number of transformations. \cite{cohen16_steer_cnns, jacobsen17_dynam_steer_block_deep_resid_networ} combine steerability with neural networks. Harmonic Networks~\citep{worrall16_harmon_networ} use the complex harmonics to generate filters that are equivariant to both rotation and translation. \cite{cohen2016group} propose group convolutions, which performs convolutions using a group operation rather than translation.  More recently, \cite{cohen2018spherical, esteves2018learning} propose spherical representations of a 3D input, which are processed with convolutions on SO3 and spherical convolutions, respectively. 

Similar to this work, Polar Transform Networks~\cite{esteves2018polar} convert an image into its log polar form to gain equivariance to rotation and scaling, while obtaining translation invariance through a spatial transformer network. We adopt a similar spatial transformer network to predict a landmark in each image, and apply the temporal normalization transform to obtain invariance to motion from optical flow.
\begin{figure*}[t]
\centering
\includegraphics[width=\linewidth]{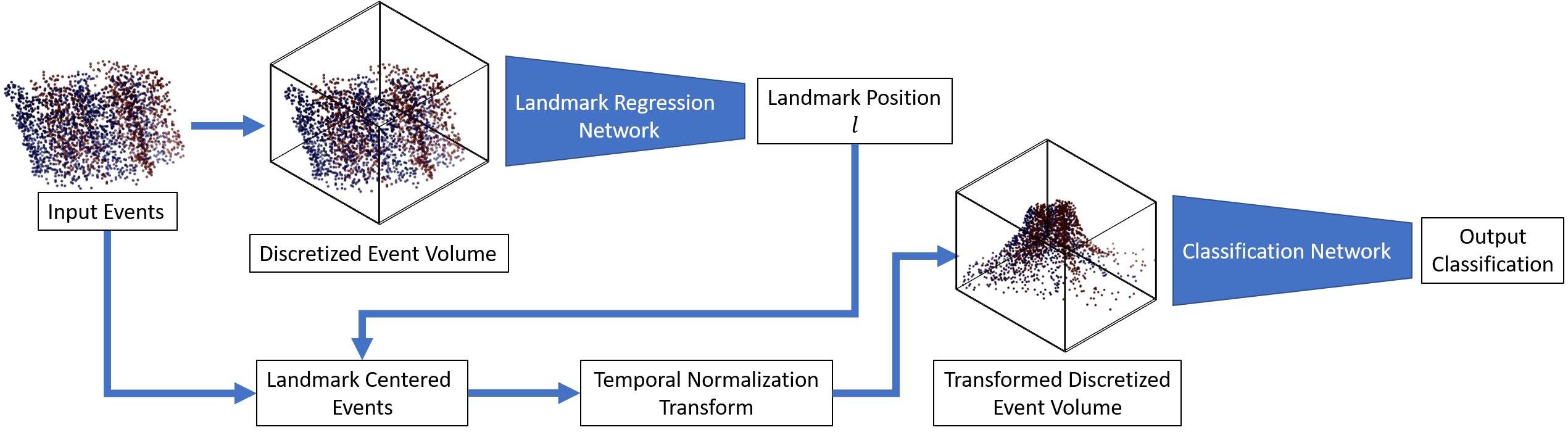}
\caption{Overview of the proposed pipeline. The input events are first converted into a discretized event volume, and passed through the landmark regression network to estimate the landmark position, $l$. This is used to center the events around $l$, on which the temporal normalization transform is applied. A second discretized event volume is generated on the transformed, centered events, and finally passed through the classification network to generate the final output classification.}
\label{fig:pipeline}
\end{figure*}
\section{Event Cameras}
\label{sec:event_cameras}
Event cameras are asynchronous cameras that trigger at changes in log image intensity. That is, when the log intensity over any pixel changes over a given threshold, $C$: $\log(I(x,y,t_1) - \log(I(x, y, t_0) > C$, the camera immediately sends an event, $e:=(x, y, t, p)$, consisting of the $x, y$ pixel position of the change, timestamp, $t$, of the change, accurate to tens of microseconds, and binary polarity, $p$, indicating whether the change was positive or negative. As events are triggered based on high changes in log intensity, they are predominantly generated on pixels with high image gradient.

In this work, we assume that the optical flow in a given spatiotemporal event window is constant. This assumption has seen success in a number of event camera works such as \cite{benosman2014event, zhu2017event}, as, due to the high temporal resolution of the events, it is usually possible to select a small enough spatiotemporal window such that this is the case. However, this assumption may not hold as strictly for classification type tasks of large or nonrigid objects, and we leave the exploration of this space to future work.

Consider a high gradient point in the image, $\mathbf{x}$ with pixel position $(x_0, y_0)$, at time $t_0=0$, moving with constant optical flow $\dot{\mathbf{x}}$. The position of each event, $e_i=(\mathbf{x}_i, t_i, p_i)$, generated at time $t_i$ can be computed as: $\mathbf{x}_i=\mathbf{x}_0+\dot{\mathbf{x}}t_i$. 
We model this with a transformation, $L_{OF}$, operating on the spatiotemporal x-y-t space.
\begin{align}
L_{OF}\begin{pmatrix}\mathbf{x}\\t\end{pmatrix}=&\begin{pmatrix}\mathbf{x}+\dot{\mathbf{x}}t\\t\end{pmatrix}
\end{align}
This is equivalent to a 3D shear deformation.
\section{Method}

\subsection{Optical Flow Equivariance for Events}
\label{sec:event_equivariance}
In this work, we model a set of events as a function, $E$, mapping from the spatiotemporal domain to polarities: 
\begin{align}
E :& (\mathbf{x}, t)\rightarrow \{-1, 0, 1\}\\
E(\mathbf{x}_i, t_i)=&\left\{\begin{array}{cc}p_i & \text{if an event was triggered at }(\mathbf{x}_i,t_i)\\ 0 & \text{otherwise} \end{array}\right.
\end{align}

\begin{prop}\label{prop:1}
The optical flow motion model $L_{OF}$ is not equivariant to 2D or 3D convolutions.
\end{prop}

\begin{proof}
For $L_{OF}$ to be equivariant to 2D or 3D convolutions, the following must be true:
\begin{align}
((L_{OF}E)\ast\phi)(\mathbf{x},t)=&L_{OF}(E\ast\phi)(\mathbf{x},t)
\end{align}
Expanding the LHS:\footnote{Here we use the equation for correlation instead of convolution. The proof holds true for both cases, but correlation is the standard form used in many deep learning frameworks.}
\begin{align}
((&L_{OF}E)\ast \phi)(\mathbf{x}, t)\nonumber\\
=&\int\limits_{\mathbf{\xi}\in\mathbb{R}^2, \tau\in\mathbb{R}}L_{OF}E(\mathbf{\xi}, \tau)\phi(\mathbf{\xi}-\mathbf{x}, \tau-t)d\mathbf{\xi}d\tau\\
=&\int\limits_{\mathbf{\xi}\in\mathbb{R}^2, \tau\in\mathbb{R}} E(\mathbf{\xi}+\dot{\mathbf{x}}\tau, \tau)\phi(\mathbf{\xi}-\mathbf{x}, \tau-t) d\mathbf{\xi}d\tau
\intertext{Applying the variable substitution: $\tau'=\tau$, $\mathbf{\xi}':=\mathbf{\xi}+\dot{\mathbf{x}}\tau'$.}
d\mathbf{\xi}&d\tau=d\mathbf{\xi}'d\tau'\\
((&L_{OF}E)\ast \phi)(\mathbf{x}, t)\nonumber\\
=&\int\limits_{\substack{\mathbf{\xi}'\in\mathbb{R}^2\\\tau'\in\mathbb{R}}}E(\mathbf{\xi}', \tau')\phi(\mathbf{\xi}'-(\mathbf{x}+\dot{\mathbf{x}}\tau'), \tau'-t)d\mathbf{\xi}'d\tau'
\end{align}
This is equivalent to the $\mathbf{x}$ term in the filter being shifted by $\dot{\mathbf{x}}\tau'$. However, as $\tau'$ is the integrand, it  will be summed over and does not appear in the output. Therefore, this operation is not equivalent to the RHS, where the output is transformed by $L_{OF}$.
\begin{align}
L&_{OF}((E\ast \phi))(\mathbf{x}, t)\nonumber\\
=&(E\ast\phi)(\mathbf{x}+\dot{\mathbf{x}}t, t)\\
=&\int\limits_{\substack{\mathbf{\xi}'\in\mathbb{R}^2\\\tau'\in\mathbb{R}}}E(\mathbf{\xi}', \tau')\phi(\mathbf{\xi}'-(\mathbf{x}+\dot{\mathbf{x}}t), \tau'-t)d\mathbf{\xi}'d\tau'
\end{align}
In the 2D case, equivariance is lost as the optical flow transformation is a 3D operation, which cannot be applied to the 2D output activations of the convolution.
\end{proof}
\subsection{The Temporal Normalization Transform}
We propose the temporal normalization transform, $\rho$, a novel coordinate transformation on the event coordinates which eliminates the dependence on the timestamps, $t$, resulting in an input that is equivariant to both 2D and 3D convolution. The transformation scales the pixel positions of the events by the reciprocal of the timestamps
\begin{align}
\rho : (\mathbf{x}, t)\rightarrow (\mathbf{x}_\rho, t_\rho)=\left(\frac{\mathbf{x}}{t}, t\right)
\end{align}

An example of this transformation can be found in Fig.~\ref{fig:transformation}, where a change in the optical flow direction transforms into a corresponding translation in the transformed space. Note that this transform produces extremely high values of $\mathbf{x}$ as $t\rightarrow 0$. We discuss this issue in \ref{sec:implementation}.

\begin{prop}\label{prop:2}
$\rho(L_{OF})$ is equivariant to both 2D and 3D convolutions. A change in optical flow in $L_{OF}$ is converted to a translation in $\rho(L_{OF})$.
\end{prop}
\begin{proof}
\begin{align}
((L&_{OF}E(\rho))\ast \phi)(\mathbf{x}, t)\nonumber\\
=&\int\limits_{\substack{\mathbf{\xi}_\rho\in\mathbb{R}^2\\\tau_\rho\in\mathbb{R}}} E(\mathbf{\xi}_\rho+\dot{\mathbf{x}}, \tau_\rho)\phi(\xi_\rho-\mathbf{x}, \tau_\rho-t) d\mathbf{\xi}_\rho d\tau_\rho
\intertext{Applying the variable substitution: $\tau':=\tau_\rho, \xi':=\xi_\rho+\dot{\mathbf{x}}$.}
d\xi_\rho&d\tau_\rho=d\xi'd\tau'\\
((L&_{OF}E(\rho))\ast \phi)(\mathbf{x}, t)\nonumber\\
=&\int\limits_{\substack{\mathbf{\xi}'\in\mathbb{R}^2\\\tau'\in\mathbb{R}}} E\left(\mathbf{\xi}', \tau'\right)\phi(\xi'-(\mathbf{x}+\dot{x}), \tau'-t) d\mathbf{\xi}'d\tau'\\
=&L_{OF}(E(\rho)\ast \phi)(\mathbf{x}, t)
\end{align}
In other words, $\rho$ converts the optical flow transformation to pure translation in 2D, which is equivariant for convolutions. A similar proof can be written for the 2D case by substituting $\tau$ for $t$ in the above equation.
\end{proof}
\subsection{Landmark Regression}
While the proposed coordinate transformation allows for convolutions that are equivariant to motion from optical flow, convolutions in this new coordinate frame are no longer equivariant to translation. 

Following the work by~\cite{esteves2018polar}, we apply a Spatial Transformer Network~\cite{jaderberg2015spatial} to regress the position of a fixed landmark, $l$ in each set of events, in order for our network to gain translation invariance. However, while the point predicted in~\cite{esteves2018polar} is exactly the origin of the target, our network is agnostic to the exact location of the landmark on the target, as long as it is consistent amongst all targets of the class. 

The landmark regression network consists of three 3x3 convolution layers followed by a 1x1 convolution layer, generating a 2D heatmap. The centroid of this heatmap is then used to offset the events.

\begin{prop}\label{prop:3}
The motion scaling transformation $\rho$ is translation invariant to convolutions after centering all events around a common landmark on the object. The position of this center is arbitrary as long as it is consistent between objects.
\end{prop}

\begin{proof}
Refer to Appendix~\ref{sec:prop3_proof}.
\end{proof}
\subsection{Input Representation}
However, standard convolutional neural network architectures perform discrete convolution, rather than the continuous integrals presented in Sec.~\ref{sec:event_equivariance}. In addition, the high temporal resolution of the event timestamps would require an extremely fine discretization along the temporal dimension to fully capture the full temporal resolution.

We address both issues by adopting the discretized event volume proposed by~\cite{zhu2018unsupervised}, which represents a set of events in a 3D volume, where events are inserted into the volume in a linearly weighted manner, as in linear interpolation.

For a set of $N$ events, $\{(x_i, y_i, t_i, p_i), i\in [1, N]\}$, the time dimension is discretized along $B$ bins. The timestamps of the events is then scaled to the range $[0, B-1]$, and the event volume is defined as:
\begin{align}
t^*_i =& (B-1)(t_i - t_1) / (t_{N} - t_1)\\
V(x,y,t)=&\sum_{i} p_i k_b(x-x_i)k_b(y-y_i)k_b(t-t^*_i)\\
k_b(a) =& \max(0, 1-|a|)
\end{align}
where $k_b(a)$ is the linear sampling kernel defined in \cite{jaderberg2015spatial}.

This representation gives us a fixed size discretization of the spatiotemporal domain, while largely retaining the full distribution of the events without rounding. In addition, due to the linearly weighted insertion of events, the mapping from events to volume is fully differentiable, which is crucial for gradient to flow from the transformed discretized event volume to the weights of the landmark regression network.
\section{Experiments}
\label{sec:experiments}

\subsection{Datasets}
Due to the relative novelty of event cameras, there does not exist the wealth of labeled data that exists for traditional images. However, several works have attempted to convert standard image-based datasets to the event space. The N-MNIST and N-CALTECH101 datasets~\cite{orchard2015converting} convert the standard MNIST handwritten digit and CALTECH 101 object classficiation datasets to events by recording the images on a screen with a moving event camera. Similarly, the CIFAR10-DVS dataset~\cite{li2017cifar10} generates events by recording moving images from the CIFAR10 dataset with an event camera. All of these datasets consist of a small (3-4) set of predefined motions, which are repeated in the train and test sets. As a result, a network is able to memorize each motion at training time, and achieve state of the art results~\cite{iyer2018neuromorphic}.

In addition, the N-CARS dataset~\cite{sironi2018hats} consists of a number of real recordings from a driving scene, with cropped events of both cars and backgrounds. However, many of these scenes do not fulfil our assumption of constant optical flow, and so we omit it from our evaluation.

\subsubsection{The N-MOVING-MNIST Dataset}
To combat the lack of optical flow diversity, we perform experiments by training a network on only a single motion from a dataset, and testing on all of the proposed motions. However, this is still limited to two additional motions for the N-MNIST dataset, for example. In addition, we have generated the N-MOVING-MNIST dataset, generated with the Event Camera Simulator~\cite{mueggler2017event}. For each digit in the MNIST test set, we simulate 30 sets of events, each with different optical flow directions (one every 12 degrees). Simulations are performd by placing a virtual camera parallel to the MNIST sample, and generating trajectories which are also parallel to the MNIST image plane, in order to generate constant optical flow values. The magnitude of the camera velocity is maintained constant through all simulations. We perform our experiments by training on the N-MNIST dataset and testing on N-MOVING-MNIST.
\begin{table*}[t]
\centering
\begin{tabular}{c|ccccc} 
train/test sets & all/all & 1/all & 1/train & all/sim & 1/sim \\\hline
Baseline & \textbf{0.991} & 0.437 & 0.442 & 0.396 & 0.207 \\
TNT & 0.981 & 0.468 & 0.464 & \textbf{0.592} & 0.318\\
TNT+regress & 0.981 & \textbf{0.485} & \textbf{0.481} & 0.566 & \textbf{0.324}
\end{tabular}
\caption{Results from experiments on N-MNIST and N-MOVING-MNIST. Experiments are denoted by train$\_$set/test$\_$set, with the following labels: 1 - trained on the first motion in the N-MNIST train set. all - training/testing on all 3 motions in the N-MNIST train/test set respectively. train - testing on all 3 motions in the N-MNIST train set. sim - testing on all motions in the N-MOVING-MNIST test set.}
\label{tab:results_table}
\end{table*}

\subsection{Network Architecture}
The goal of our experiments is to demonstrate the proposed method's ability to reduce the capacity of the network needed to memorize different motions. To achieve this, we perform our experiments with only a small CNN with limited capacity. The classification network takes as input the discretized event volume of the events after TNT, and consists of two convolution layers with relu activations, the first of which has stride 2, and followed by average pooling after each layer with stride 2. The activations are then passed through two fully connected layers, outputting 1024 feature channels and finally N output classes as a one-hot vector. 
\subsection{Implementation Details}
\label{sec:implementation}
All models are trained for 60,000 iterations with a batch size of 64, and saturated validation accuracy before stopping. When training with the landmark regressor, random translations are applied as data augmentation.

One issue for implementation is the tendency for the transformed coordinates to grow towards infinity as $t\rightarrow 0$. Due to the need to discretize the spatial dimension at each convolution layer, it is prohibitively expensive to try to encompass all transformed events when discretizing. Instead, we have chosen to omit any points that fall outside a predefined image boundary, $[W, H]$. For this work, we have kept the same transformed image size as the original input image. In addition, the number of events falling out of the image can be controlled by scaling the timestamps before applying the transform, as equivariance is maintained for any constant scaling of the timestamps. However, there is a tradeoff between minimizing the number of events leaving the image and minimizing the compression of events at the highest timestamps. Due to the discretization in the spatial domain, transformed events which are very close together will be placed in the same voxel of the volume. In practice, we found that scaling the timestamps to be between 0 and $B-1$, where $B=9$ is the number of bins, works well. With this scaling, only events in the first bin may be transformed out of the image. 
\subsection{Evaluation Details}
In all experiments, we compare the proposed method with a baseline network which takes as input a discretized event volume generated directly from the raw events. In addition, we perform an ablation study on the landmark regression network (TNT+regress) with a heuristic that centers events around the center of the image (TNT). Several experiments are performed with N-MNIST and N-MOVING-MNIST, where each network is trained on either all provided motions in N-MNIST or only the first motion, and then tested on either all training or test motions in N-MNIST or all motions in N-MOVING-MNIST. Each sequence in N-MNIST consists of three motions, and each individual motion is extracted by thresholding events by the timestamps at which each motion starts and stops.
\subsection{Results}
The results from all experiments can be found in Table~\ref{tab:results_table}. 

All models perform well when trained and tested on all motions. However, performance drops significantly when restricting the number of optical flow directions at training compared to those seen at test time. In particular, even when the training set is used at test time, but with more motions, the network still suffers from a significant drop in accuracy. In addition, test accuracy decreases as the differential in number of motions seen between training and test sets increases. This demonstrates the need for methods that are robust to motion, as it is difficult for any given training set to encompass all expected motions to be seen at test time in a natural scene. 

The proposed method outperforms the baseline in all experiments where the number of train motions are limited. In both data constrained experiments, the landmark regression network provides a boost in performance, although the transform network with a center heuristic outperforms the landmark regression network in the all/sim experiment.
\subsection{Discussion}
The sharp difference in classification accuracy between the all/all and 1/all tests highlight the factors ignored when testing on the same types of motion as seen at training, and demonstrate the need for methods that are robust to differences in motion, such as the proposed equivariant method. 

The landmark regression network only provides a modest boost in performance over the center heuristic, but this is likely due to the fact that, due to the nature of the datasets which are cropped around the object, the center of the image is usually close to being a good landmark for all targets. As it looks like the transform is relatively robust, one could explore in future work combining a region proposal network in a detection and classification scheme, where the center of the bounding box predicted by the RPN is used.

However, there are still a number of issues that could improve the accuracy of our method (beyond having a deeper network). 
Besides optical flow, there are several other problems that are exclusive to events for tasks such as object classification, stemming from the fact that event cameras capture changes in image intensity. One issue is that the polarity of each event is itself dependent on the direction of the motion. For example, reversing the direction of motion over an edge in the image would generate events of the opposite polarity. While it is possible to ignore the polarities altogether, the relative differences between polarities still provide useful information for classification (rising vs falling edges in the image, for example). Currently, we rely on the network to learn these differences automatically. Another issue is that edges parallel to the direction of motion do not trigger changes in the image, and so do not generate any events. The result is that different parts of the image may be missing from the events, depending on the direction of motion.

\section{Conclusions}
In this work, we have proposed a novel coordinate transformation for events which, when combined with centering from a learned landmark, generates a representation that is equivariant to optical flow and invariant to translation for CNNs. We demonstrate that this transformation improves CNN classification performance in data constrained regimes with limited motion directions. A key take away of this work is that the transformation of events due to motion must be taken into account in order to achieve optimal performance in a deep learning regime. In future work, we hope to relax the global optical flow assumption this constraint, by relaxing to a local transformation which only assumes local consistency in flow.
\section{Acknowledgements}
Thanks to Tobi Delbruck and the team at iniLabs and iniVation for providing and supporting the DAVIS-346b cameras. This work was supported in part by the Semiconductor Research Corporation (SRC)
and DARPA. We also gratefully appreciate support through the following grants: NSF-DGE-0966142 (IGERT), NSF-IIP-1439681 (I/UCRC), NSF-IIS-1426840, NSF-IIS-1703319, NSF MRI 1626008, ARL RCTA W911NF-10-2-0016, ONR N00014-17-1-2093, the Honda Research Institute and the DARPA FLA program.
\appendix
\section{Proof of Proposition~\ref{prop:3}}
\label{sec:prop3_proof}
Let $\mathbf{c}\in\mathbb{R}^2$ be the pixel position of the landmark position for a given set of events, and $\mathbf{s}\in\mathbb{R}^2$ be a translation of the events. Given an accurate landmark regression network, the predicted landmark position is $\mathbf{l}=\mathbf{c}+\mathbf{s}$. Let a translation of the events be represented by the function $T_\mathbf{k} : (\mathbf{x}, t)\rightarrow(\bar{\mathbf{x}}, \bar{t})=(\mathbf{x}+\mathbf{k}, t)$.
\begin{align}
((\rho& T_{\mathbf{s}-\mathbf{l}}L_{OF}E)\ast \phi)(\mathbf{x}, \tau)\nonumber\\
=&\int\limits_{\mathbf{\xi}\in\mathbb{R}^2, \tau\in\mathbb{R}} E\left(\frac{\mathbf{\xi}+\mathbf{s}-\mathbf{l}}{\tau}+\dot{\mathbf{x}}, \tau\right)\phi(\mathbf{\xi}-\mathbf{x}, \tau) d\mathbf{\xi}d\tau\\
=&\int\limits_{\mathbf{\xi}_\rho\in\mathbb{R}^2, \tau_\rho\in\mathbb{R}} E\left(\mathbf{\xi}_\rho-\frac{\mathbf{c}}{\tau_\rho}+\dot{\mathbf{x}}, \tau_\rho\right)\phi(\mathbf{\xi}_\rho-\mathbf{x}, \tau_\rho) d\mathbf{\xi}_\rho d\tau_\rho
\intertext{Applying the variable substitution: $\tau':=\tau$, $\xi':=\xi_\rho-\frac{\mathbf{c}}{\tau'}+\dot{\mathbf{x}}$.}
=&\int\limits_{\mathbf{\xi}'\in\mathbb{R}^2, \tau'\in\mathbb{R}} E(\xi', \tau')\phi\left(\mathbf{\xi}'-\left(\mathbf{x}-\frac{\mathbf{c}}{\tau'}+\dot{\mathbf{x}}\right), \tau'\right) d\mathbf{\xi}'d\tau'\\
=&\rho T_{-\mathbf{c}}L_{OF}(E(\rho)\ast \phi)(\mathbf{x}, \tau)
\end{align}
As this is true for any $\mathbf{s}$, the transformation is translation invariant. This is true for any common landmark $\mathbf{c}$.
\bibliography{refs}

\begin{thebibliography}{32}
\providecommand{\natexlab}[1]{#1}
\providecommand{\url}[1]{\texttt{#1}}
\expandafter\ifx\csname urlstyle\endcsname\relax
  \providecommand{\doi}[1]{doi: #1}\else
  \providecommand{\doi}{doi: \begingroup \urlstyle{rm}\Url}\fi

\bibitem[Alonso \& Murillo(2018)Alonso and Murillo]{alonso2018ev}
Alonso, I. and Murillo, A.~C.
\newblock Ev-segnet: Semantic segmentation for event-based cameras.
\newblock \emph{arXiv preprint arXiv:1811.12039}, 2018.

\bibitem[Amir et~al.()Amir, Taba, Berg, Melano, McKinstry, Di~Nolfo, Nayak,
  Andreopoulos, Garreau, Mendoza, et~al.]{amir2017low}
Amir, A., Taba, B., Berg, D., Melano, T., McKinstry, J., Di~Nolfo, C., Nayak,
  T., Andreopoulos, A., Garreau, G., Mendoza, M., et~al.
\newblock A low power, fully event-based gesture recognition system.

\bibitem[Benosman et~al.(2014)Benosman, Clercq, Lagorce, Ieng, and
  Bartolozzi]{benosman2014event}
Benosman, R., Clercq, C., Lagorce, X., Ieng, S.-H., and Bartolozzi, C.
\newblock Event-based visual flow.
\newblock \emph{IEEE Trans. Neural Netw. Learning Syst.}, 25\penalty0
  (2):\penalty0 407--417, 2014.

\bibitem[Cohen \& Welling(2016{\natexlab{a}})Cohen and
  Welling]{cohen16_steer_cnns}
Cohen, T.~S. and Welling, M.
\newblock Steerable cnns.
\newblock 2016{\natexlab{a}}.
\newblock URL \url{http://arxiv.org/abs/1612.08498v1}.

\bibitem[Cohen \& Welling(2016{\natexlab{b}})Cohen and Welling]{cohen2016group}
Cohen, T.~S. and Welling, M.
\newblock Group equivariant convolutional networks.
\newblock \emph{arXiv preprint arXiv:1602.07576}, 2016{\natexlab{b}}.

\bibitem[Cohen et~al.(2018)Cohen, Geiger, K{\"o}hler, and
  Welling]{cohen2018spherical}
Cohen, T.~S., Geiger, M., K{\"o}hler, J., and Welling, M.
\newblock Spherical cnns.
\newblock \emph{arXiv preprint arXiv:1801.10130}, 2018.

\bibitem[Esteves et~al.(2018{\natexlab{a}})Esteves, Allen-Blanchette, Makadia,
  and Daniilidis]{esteves2018learning}
Esteves, C., Allen-Blanchette, C., Makadia, A., and Daniilidis, K.
\newblock Learning so (3) equivariant representations with spherical cnns.
\newblock In \emph{European Conference on Computer Vision}, pp.\  54--70.
  Springer, 2018{\natexlab{a}}.

\bibitem[Esteves et~al.(2018{\natexlab{b}})Esteves, Allen-Blanchette, Zhou, and
  Daniilidis]{esteves2018polar}
Esteves, C., Allen-Blanchette, C., Zhou, X., and Daniilidis, K.
\newblock Polar transformer networks.
\newblock In \emph{International Conference on Learning Representations},
  2018{\natexlab{b}}.
\newblock URL \url{https://openreview.net/forum?id=HktRlUlAZ}.

\bibitem[Ferraro \& Caelli(1988)Ferraro and Caelli]{ferraro1988relationship}
Ferraro, M. and Caelli, T.~M.
\newblock Relationship between integral transform invariances and lie group
  theory.
\newblock \emph{JOSA A}, 5\penalty0 (5):\penalty0 738--742, 1988.

\bibitem[Freeman et~al.(1991)Freeman, Adelson, et~al.]{freeman1991design}
Freeman, W.~T., Adelson, E.~H., et~al.
\newblock The design and use of steerable filters.
\newblock \emph{IEEE Transactions on Pattern analysis and machine
  intelligence}, 13\penalty0 (9):\penalty0 891--906, 1991.

\bibitem[Iacono et~al.(2018)Iacono, Weber, Glover, and
  Bartolozzi]{iacono2018towards}
Iacono, M., Weber, S., Glover, A., and Bartolozzi, C.
\newblock Towards event-driven object detection with off-the-shelf deep
  learning.
\newblock In \emph{2018 IEEE/RSJ International Conference on Intelligent Robots
  and Systems (IROS)}, pp.\  1--9. IEEE, 2018.

\bibitem[Iyer et~al.(2018)Iyer, Chua, and Li]{iyer2018neuromorphic}
Iyer, L.~R., Chua, Y., and Li, H.
\newblock Is neuromorphic mnist neuromorphic? analyzing the discriminative
  power of neuromorphic datasets in the time domain.
\newblock \emph{arXiv preprint arXiv:1807.01013}, 2018.

\bibitem[Jacobsen et~al.(2017)Jacobsen, Brabandere, and
  Smeulders]{jacobsen17_dynam_steer_block_deep_resid_networ}
Jacobsen, J.-H., Brabandere, B.~d., and Smeulders, A. W.~M.
\newblock Dynamic steerable blocks in deep residual networks.
\newblock \emph{CoRR}, 2017.
\newblock URL \url{http://arxiv.org/abs/1706.00598v2}.

\bibitem[Jaderberg et~al.(2015)Jaderberg, Simonyan, Zisserman,
  et~al.]{jaderberg2015spatial}
Jaderberg, M., Simonyan, K., Zisserman, A., et~al.
\newblock Spatial transformer networks.
\newblock In \emph{Advances in Neural Information Processing Systems}, pp.\
  2017--2025, 2015.

\bibitem[Lagorce et~al.(2017)Lagorce, Orchard, Galluppi, Shi, and
  Benosman]{lagorce2017hots}
Lagorce, X., Orchard, G., Galluppi, F., Shi, B.~E., and Benosman, R.~B.
\newblock Hots: a hierarchy of event-based time-surfaces for pattern
  recognition.
\newblock \emph{IEEE transactions on pattern analysis and machine
  intelligence}, 39\penalty0 (7):\penalty0 1346--1359, 2017.

\bibitem[Li et~al.(2017)Li, Liu, Ji, Li, and Shi]{li2017cifar10}
Li, H., Liu, H., Ji, X., Li, G., and Shi, L.
\newblock Cifar10-dvs: an event-stream dataset for object classification.
\newblock \emph{Frontiers in neuroscience}, 11:\penalty0 309, 2017.

\bibitem[Maqueda et~al.(2018)Maqueda, Loquercio, Gallego, Garc{\'\i}a, and
  Scaramuzza]{maqueda2018event}
Maqueda, A.~I., Loquercio, A., Gallego, G., Garc{\'\i}a, N., and Scaramuzza, D.
\newblock Event-based vision meets deep learning on steering prediction for
  self-driving cars.
\newblock In \emph{Proceedings of the IEEE Conference on Computer Vision and
  Pattern Recognition}, pp.\  5419--5427, 2018.

\bibitem[Moeys et~al.(2016)Moeys, Corradi, Kerr, Vance, Das, Neil, Kerr, and
  Delbr{\"u}ck]{moeys2016steering}
Moeys, D.~P., Corradi, F., Kerr, E., Vance, P., Das, G., Neil, D., Kerr, D.,
  and Delbr{\"u}ck, T.
\newblock Steering a predator robot using a mixed frame/event-driven
  convolutional neural network.
\newblock In \emph{Event-based Control, Communication, and Signal Processing
  (EBCCSP), 2016 Second International Conference on}, pp.\  1--8. IEEE, 2016.

\bibitem[Mueggler et~al.(2017)Mueggler, Rebecq, Gallego, Delbruck, and
  Scaramuzza]{mueggler2017event}
Mueggler, E., Rebecq, H., Gallego, G., Delbruck, T., and Scaramuzza, D.
\newblock The event-camera dataset and simulator: Event-based data for pose
  estimation, visual odometry, and slam.
\newblock \emph{The International Journal of Robotics Research}, 36\penalty0
  (2):\penalty0 142--149, 2017.

\bibitem[Orchard et~al.(2015)Orchard, Jayawant, Cohen, and
  Thakor]{orchard2015converting}
Orchard, G., Jayawant, A., Cohen, G.~K., and Thakor, N.
\newblock Converting static image datasets to spiking neuromorphic datasets
  using saccades.
\newblock \emph{Frontiers in neuroscience}, 9:\penalty0 437, 2015.

\bibitem[Qi et~al.(2017)Qi, Su, Mo, and Guibas]{qi2017pointnet}
Qi, C.~R., Su, H., Mo, K., and Guibas, L.~J.
\newblock Pointnet: Deep learning on point sets for 3d classification and
  segmentation.
\newblock \emph{Proc. Computer Vision and Pattern Recognition (CVPR), IEEE},
  1\penalty0 (2):\penalty0 4, 2017.

\bibitem[Segman et~al.(1992)Segman, Rubinstein, and Zeevi]{segman1992canonical}
Segman, J., Rubinstein, J., and Zeevi, Y.~Y.
\newblock The canonical coordinates method for pattern deformation: Theoretical
  and computational considerations.
\newblock \emph{IEEE Transactions on Pattern Analysis and Machine
  Intelligence}, 14\penalty0 (12):\penalty0 1171--1183, 1992.

\bibitem[Sekikawa et~al.(2018)Sekikawa, Ishikawa, Hara, Yoshida, Suzuki, Sato,
  and Saito]{sekikawa2018constant}
Sekikawa, Y., Ishikawa, K., Hara, K., Yoshida, Y., Suzuki, K., Sato, I., and
  Saito, H.
\newblock Constant velocity 3d convolution.
\newblock In \emph{2018 International Conference on 3D Vision (3DV)}, pp.\
  343--351. IEEE, 2018.

\bibitem[Simoncelli et~al.(1992)Simoncelli, Freeman, Adelson, and
  Heeger]{simoncelli1992shiftable}
Simoncelli, E.~P., Freeman, W.~T., Adelson, E.~H., and Heeger, D.~J.
\newblock Shiftable multiscale transforms.
\newblock \emph{IEEE transactions on Information Theory}, 38\penalty0
  (2):\penalty0 587--607, 1992.

\bibitem[Sironi et~al.()Sironi, Brambilla, Bourdis, Lagorce, and
  Benosman]{sironi2018hats}
Sironi, A., Brambilla, M., Bourdis, N., Lagorce, X., and Benosman, R.
\newblock Hats: Histograms of averaged time surfaces for robust event-based
  object classification.

\bibitem[Teo \& Hel-Or(1998)Teo and Hel-Or]{teo1998design}
Teo, P.~C. and Hel-Or, Y.
\newblock Design of multi-parameter steerable functions using cascade basis
  reduction.
\newblock In \emph{Computer Vision, 1998. Sixth International Conference on},
  pp.\  187--192. IEEE, 1998.

\bibitem[Wang et~al.(2019)Wang, Zhang, Yuan, and Lu]{wang2019space}
Wang, Q., Zhang, Y., Yuan, J., and Lu, Y.
\newblock Space-time event clouds for gesture recognition: from rgb cameras to
  event cameras.
\newblock \emph{IEEE Winter Conference on Applications of Computer Vision},
  2019.

\bibitem[Worrall et~al.(2016)Worrall, Garbin, Turmukhambetov, and
  Brostow]{worrall16_harmon_networ}
Worrall, D.~E., Garbin, S.~J., Turmukhambetov, D., and Brostow, G.~J.
\newblock Harmonic networks: Deep translation and rotation equivariance.
\newblock \emph{arXiv preprint arXiv:1612.04642}, 2016.

\bibitem[Ye et~al.(2018)Ye, Mitrokhin, Parameshwara, Ferm{\"u}ller, Yorke, and
  Aloimonos]{ye2018unsupervised}
Ye, C., Mitrokhin, A., Parameshwara, C., Ferm{\"u}ller, C., Yorke, J.~A., and
  Aloimonos, Y.
\newblock Unsupervised learning of dense optical flow and depth from sparse
  event data.
\newblock \emph{arXiv preprint arXiv:1809.08625}, 2018.

\bibitem[Zhu et~al.(2018{\natexlab{a}})Zhu, Yuan, Chaney, and
  Daniilidis]{Zhu-RSS-18}
Zhu, A., Yuan, L., Chaney, K., and Daniilidis, K.
\newblock Ev-flownet: Self-supervised optical flow estimation for event-based
  cameras.
\newblock In \emph{Proceedings of Robotics: Science and Systems}, Pittsburgh,
  Pennsylvania, June 2018{\natexlab{a}}.
\newblock \doi{10.15607/RSS.2018.XIV.062}.

\bibitem[Zhu et~al.(2017)Zhu, Atanasov, and Daniilidis]{zhu2017event}
Zhu, A.~Z., Atanasov, N., and Daniilidis, K.
\newblock Event-based feature tracking with probabilistic data association.
\newblock In \emph{Robotics and Automation (ICRA), 2017 IEEE International
  Conference on}, pp.\  4465--4470. IEEE, 2017.

\bibitem[Zhu et~al.(2018{\natexlab{b}})Zhu, Yuan, Chaney, and
  Daniilidis]{zhu2018unsupervised}
Zhu, A.~Z., Yuan, L., Chaney, K., and Daniilidis, K.
\newblock Unsupervised event-based learning of optical flow, depth, and
  egomotion.
\newblock \emph{arXiv preprint arXiv:1812.08156}, 2018{\natexlab{b}}.

\end{thebibliography}
\bibliographystyle{icml2019}

\end{document}